\documentclass[11pt]{article}
\usepackage{fullpage}
\usepackage{amsmath,amsthm,amsfonts,amssymb,dsfont}
\usepackage{microtype}
\usepackage{booktabs}
\usepackage{hyperref,cleveref}

\Crefformat{equation}{#2(#1)#3}
\crefformat{equation}{#2(#1)#3}
\Crefrangeformat{equation}{#3(#1)#4--#5(#2)#6}
\crefrangeformat{equation}{#3(#1)#4--#5(#2)#6}
\Crefmultiformat{equation}{#2(#1)#3}{ and~#2(#1)#3}{, #2(#1)#3}{ and~#2(#1)#3}
\crefmultiformat{equation}{#2(#1)#3}{ and~#2(#1)#3}{, #2(#1)#3}{ and~#2(#1)#3}

\usepackage{enumitem}
\setlist[description]{font=\normalfont}

\usepackage{csquotes}
\MakeOuterQuote{"}

\usepackage{mathtools}
\mathtoolsset{centercolon}
\DeclarePairedDelimiterXPP\ind[1]{\mathds{1}}{\lbrace}{\rbrace}{}{#1} 
\DeclarePairedDelimiterX\eval[1]{\lbrace}{\rvert}{#1 \delimsize\rbrace} 
\DeclarePairedDelimiter\abs{\lvert}{\rvert} 
\DeclarePairedDelimiter\card{\lvert}{\rvert} 
\DeclarePairedDelimiter\ceil{\lceil}{\rceil} 
\DeclarePairedDelimiter\floor{\lfloor}{\rfloor} 
\DeclarePairedDelimiter\del{\lparen}{\rparen} 
\DeclarePairedDelimiter\sbr{\lbrack}{\rbrack} 
\DeclarePairedDelimiter\set{\lbrace}{\rbrace} 
\DeclarePairedDelimiter\intoc{\lparen}{\rbrack} 
\DeclarePairedDelimiter\intoo{\lparen}{\rparen} 
\DeclarePairedDelimiter\intcc{\lbrack}{\rbrack} 

\providecommand\given{}
\newcommand{\pipeseparator}{\nonscript\:\delimsize\vert\nonscript\:\mathopen{}}
\begingroup
  \catcode`\|=13
  \gdef|{\pipeseparator}
\endgroup
\newcommand{\activatepipe}{%
  \renewcommand\given\pipeseparator
  \mathcode`\|="8000
}

\DeclarePairedDelimiterX{\Set}[1]{\{}{\}}{%
  \activatepipe
  #1
}
\DeclarePairedDelimiterX\Sbr[1]\lbrack\rbrack{%
  \activatepipe
  #1
}
\DeclarePairedDelimiterX\Del[1]\lparen\rparen{%
  \activatepipe
  #1
}
\DeclarePairedDelimiterX\Braket[1]\langle\rangle{%
  \activatepipe
  #1
}

\newtheorem{theorem}{Theorem} 
\newtheorem{lemma}{Lemma} 
\newtheorem{proposition}{Proposition} 
\newtheorem{definition}{Definition}
\newtheorem{remark}{Remark}
\newenvironment{proofof}[1]{\begin{proof}[Proof of #1]}{\end{proof}}

\DeclareMathOperator\E{\mathbb{E}} 
\DeclareMathOperator\supp{supp} 
\DeclareMathOperator\poly{poly}

\def\ddefloop#1{\ifx\ddefloop#1\else\ddef{#1}\expandafter\ddefloop\fi}
\def\ddef#1{\expandafter\def\csname cal#1\endcsname{\ensuremath{\mathcal{#1}}}}
\ddefloop ABCDEFGHIJKLMNOPQRSTUVWXYZ\ddefloop

\usepackage[mathscr]{euscript}
\newcommand\bits{\set{0,1}}
\newcommand\concepts{\calC}
\newcommand\hypotheses{\calH}
\newcommand\groups{\mathscr{S}}
\newcommand\family{\calF}
\newcommand\domain{\calX}
\newcommand\dcc{\ensuremath{d_{\concepts}}}
\newcommand\dcs{\ensuremath{d_{\concepts|_S}}}
\newcommand\dhyp{\ensuremath{d_{\hypotheses}}}

\newcommand\opt{\ensuremath{\operatorname{opt}}}
\newcommand\outdeg{\ensuremath{\operatorname{outdeg}}}
\newcommand\F{\ensuremath{\mathbb{F}}}
\newcommand\kpairs{\ensuremath{k_2}}

\usepackage[round]{natbib}
\usepackage[export]{adjustbox}

\title{The price of multi-group transductive learning}
\author{Noah Bergam \\ \textsl{Columbia University} \and Samuel Deng \\ \textsl{
Columbia University} \and Daniel Hsu \\ \textsl{Columbia University}}

\begin{document}

\maketitle

\begin{abstract}
 We show every multi-group learner in the leave-one-out transductive setting may incur a surprisingly large multiplicative penalty in its error rate on some group relative to the error rate achievable in the single-group setting. Indeed, the penalty can increase linearly with the number of groups, up to roughly the square-root of the sample size. We also show that in the analogous (group-realizable) statistical setting, the penalty is always at most a poly-logarithmic factor in the sample size, independent of the number of groups.
\end{abstract}

\section{Introduction}

The \emph{multi-group} extensions of statistical and online learning frameworks address the need to consider performance of predictors on subpopulations (a.k.a.~\emph{groups} or \emph{subgroups})~\citep{blum2019advancing,rothblum2021multi,tosh2022simple}.
In these settings, multi-group learners can achieve nearly the same performance guarantees (e.g., excess risk, regret) for all groups (after suitable rescaling based on the group size) as compared to what is achievable on individual groups.
These performance guarantees hold even if (1) group-specific optimal predictors disagree where the groups intersect, and (2) the number of groups is very large or even infinite.
The statistical cost is relatively mild and typically amortized by the per-group excess risk or regret whenever the "complexity" of the family of groups is dominated by that of the performance benchmark class.
For example, in the multi-group agnostic learning scenario~\citep{rothblum2021multi,tosh2022simple}, a multi-group learner can produce a predictor $f$ such that, for all groups $S \in \groups$,
\begin{equation}
  \text{risk of $f$ on $S$}
  \, \leq \,
  \min_{h \in \hypotheses} \, \text{risk of $h$ on $S$}
  \, + \,
  O\del*{ \sqrt{ \frac{\text{complexity of $\hypotheses$ on $S$} + \text{complexity of $\groups$}}{\text{training sample size in $S$}} } }
  ,
  \label{eq:multigroup_agnostic}
\end{equation}
where $\hypotheses$ is the benchmark class and $\groups$ is the family of groups.
In some other settings, there is no dependence on $\groups$ whatsoever in the bound.
Indeed, the difficulty of multi-group learning appears to primarily be algorithmic: for example, the classical principle of Empirical Risk Minimization over the benchmark class is not always sufficient for multi-group learning; instead successful multi-group learners typically use ensemble methods or other forms of "improper learning".

In this paper, we initiate the study of the multi-group extension of another standard learning setting---leave-one-out \emph{transductive learning}~\citep{vapnik1974theory,haussler1994predicting}---and we find that, there, all multi-group learners must incur a high statistical cost compared to what is achievable in the usual (single-group) scenario.
Specifically, we show (in \Cref{sec:lb,sec:ub}) that the best achievable per-group error rate (i.e., the smallest error rate achievable when the group is considered in isolation) may be \emph{multiplied} in the multi-group scenario (at least for one of the groups) by a factor that increases linearly with the number of groups $\card{\groups}$, up to $\card{\groups} \asymp \sqrt{n}$, where $n$ is the sample size.
This stands in contrast to the effect of multiple groups in an analogous statistical learning setting, where the cost is at most a factor poly-logarithmic in the sample size and independent of $\card{\groups}$ (\Cref{sec:stat_price}).

\begin{table}[t]
\centering
\renewcommand{\arraystretch}{1.35}
\begin{tabular}{@{}lcc@{}}
\toprule
\textbf{Setting} & \textbf{Price lower bound} & \textbf{Price upper bound} \\
\midrule
Transductive
  & $\Omega\!\left(\min\{|\mathcal{S}|, \sqrt{n}\}\right)$ \ {\scriptsize [Thm.~\ref{thm:lb}]}
  & $O\!\left(\min\{|\mathcal{S}|, \sqrt{n}\}\right)$ \ {\scriptsize [Thms.~\ref{thm:sqrtn},\ref{thm:lp}]} \\
Statistical, oblivious
  & $1$
  & $O(\log n)$ \ { \scriptsize \citep{ardeshir2026realizable}, Thm.~4} \\ 
Statistical, non-oblivious
  & $\Omega(\log n)$ \ {\scriptsize [Thm.~\ref{thm:stat-lower}]}
  & $O(\log^2 n)$ \ {\scriptsize [Thm.~\ref{thm:stat-upper}]} \\
\bottomrule
\end{tabular}
\caption{Bounds on the price of multi-group learning, where, by price, we mean the worst-case ratio, over
instances and groups $S \in \mathcal{S}$, of the best achievable multi-group error
rate on $S$ to the error rate achievable when $S$ is considered in isolation. The difference between the oblivious and non-oblivious statistical settings is discussed in Section~\ref{sec:multigroup}.
}

\label{tab:summary}
\end{table}

\section{Background on statistical and transductive learning} \label{sec:preliminaries}

We introduce (leave-one-out) transductive learning in both the (standard) single-group and the multi-group settings.
For context, we also introduce single- and multi-group statistical learning, and discuss the achievable excess error rates in these settings as established in prior works.
Additional related works are discussed in \Cref{sec:related}.

\subsection{Realizable (single-group) transductive and statistical learning}
\label{sec:singlegroup}

Let $\domain$ be a data domain, $\concepts \subseteq \bits^{\domain}$ be a concept class, and $U = (x_1,\dotsc,x_n) \in \domain^n$ be $n$ (unlabeled) points from the domain.
The (realizable) \emph{transductive learning} scenario of~\citet{haussler1994predicting} for binary classification proceeds as follows \citep[see also][]{asilis2024regularization}:
\begin{enumerate}
  \item Nature chooses a target concept $c^* \in \concepts$.
  \item An index $i \in [n] \coloneqq \set{1,\dotsc,n}$ is chosen uniformly at random.
  \item The learner receives unlabeled points $U$, index $i$, and labels $y_j \coloneqq c^*(x_j)$ for all $j \neq i$.

    (We may regard $(x_j,y_j)_{j \neq i}$ as the training data, and $x_i$ as the test point.)
  \item The learner makes a (possibly randomized) prediction $\hat y_i \in \bits$ of $y_i \coloneqq c^*(x_i)$.
\end{enumerate}
The \emph{transductive error rate} of the learner is the probability that the learner's prediction differs from the label assigned to the test point by the target concept, where the probability considers the random choice of $i$ as well as any randomness used by the learner in forming its prediction:
\begin{equation*}
  \Pr\sbr*{ \hat y_i \neq y_i }
  = \frac1n \sum_{i'=1}^n \Pr\Sbr*{ \hat y_i \neq y_i \given i = i' } .
\end{equation*}
Here, $\hat y_{i'}$ is the learner's prediction when $(x_j,y_j)_{j \neq i'}$ is the training data and $x_{i'}$ is the test point.

In transductive learning, we make no distributional assumptions about the data points $U$ (in contrast to statistical learning, where $U$ is typically assumed to be an i.i.d.~sample), and instead "training" and "test" points are treated as exchangeable (in contrast to online learning, where test points always come sequentially after training points).
Algorithms for transductive learning (and online learning) may be converted into algorithms for statistical learning~\citep{haussler1994predicting,aden2023optimal,asilis2024regularization}, but the reverse is not always true.

When $U$ is assumed to be drawn i.i.d.~from a probability distribution $P$ over $\domain$, we recover the (realizable) statistical learning scenario of~\citet{valiant1984theory}.
In this case, we can always assume $i = n$; the \emph{statistical error rate} of the learner is the probability (over just the random draw of $x_n$ and any internal randomness of the learner) that $\hat y_n \neq y_n$.
(Because the statistical error rate also depends on $x_1,\dotsc,x_{n-1}$, we are typically interested in understanding how large this quantity can be either in expectation or with high probability over the choice of $x_1,\dotsc,x_{n-1}$; in this paper, we ignore this issue but refer the interested reader to works of~\citet{aden2023one,aden2023optimal} for further discussion.)

\paragraph{Achievable transductive and statistical error rates.}

In both the realizable transductive and statistical learning settings, the best achievable error rate is
\begin{equation}
  \Pr\sbr*{ \hat y_i \neq y_i } = O\del*{ \frac{\dcc}{n} }
  \label{eq:single}
\end{equation}
where $\dcc$ is the VC dimension of $\concepts$~\citep{haussler1994predicting,blumer1989learnability,ehrenfeucht1989general,hanneke2016optimal}.
In fact, the learner based on "Empirical Risk Minimization" that predicts using any $\hat c \in \concepts$ consistent with $(x_1,y_1),\dotsc,(x_{n-1},y_{n-1})$ achieves statistical error rate
\begin{equation}
  \Pr\sbr*{ \hat y_i \neq y_i } = O\del*{\frac{\dcc \log(n)}{n}}
  \label{eq:single_erm}
\end{equation}
with high probability which is optimal up to the $\log(n)$ factor.
(Without the realizability assumption, we can obtain excess error rates of $O(\sqrt{\dhyp / n})$ relative to the best benchmark hypothesis in a hypothesis class $\hypotheses \subseteq \bits^{\domain}$.)

\subsection{Multi-group transductive and statistical learning}
\label{sec:multigroup}

The natural multi-group generalization of transductive learning additionally considers a \emph{family of groups} $\groups \subseteq 2^{\domain}$ (known to the learner), and for each group $S \in \groups$, we define the \emph{transductive error rate of the learner on group $S$} to be
\begin{equation}
  \label{eq:transductive_err_multi}
  \Pr\sbr*{ x_i \in S \wedge \hat y_i \neq y_i }
  = \frac1n \sum_{i'=1}^n \Pr\Sbr*{ x_i \in S \wedge \hat y_i \neq y_i \given i = i' } .
\end{equation}
Just as in the single-group setting, the learner is prepared to make a prediction $\hat y_{i'}$ (based on $(x_j,y_j)_{j \neq i'}$) for all $i' \in [n]$; it is only the performance measure on a specific group $S$ that restricts attention to those $i'$ with $x_{i'} \in S$.

When $U$ is assumed to be drawn i.i.d.~from a probability distribution $P$ over $\domain$, we recover the multi-group statistical learning scenario~\citep{rothblum2021multi,tosh2022simple}.
Again, we can always assume $i = n$, so the \emph{statistical error rate of the learner on a group $S \subseteq \domain$} is
\begin{equation*}
  \Pr\sbr*{ x_n \in S \wedge \hat y_n \neq y_n }
  = \Pr\sbr*{ x_n \in S } \Pr\Sbr*{ \hat y_n \neq c^*(x_n) \given x_n \in S }
  .
\end{equation*}

\paragraph{Achievable multi-group statistical error rates.}

\citet{tosh2022simple} introduced the \emph{group realizability} assumption on $\concepts$ with respect to a family of groups $\groups$ and a \emph{hypothesis class} $\hypotheses \subseteq \bits^{\domain}$: it requires that for each concept $c \in \concepts$ and each group $S \in \groups$, there exists some hypothesis $h_{c,S} \in \hypotheses$ such that $c$ and $h_{c,S}$ agree on $\supp(P) \cap S$ (i.e., all $x$ in both the support of $P$ and $S$).
Note that this implies that the restriction $\concepts|_S$ of $\concepts$ to $S$ is a subset of the restriction $\hypotheses|_S$ of $\hypotheses$ to $S$.
It is important to note that the strength of the group realizability assumption increases as $\groups$ becomes larger.

In Section \ref{sec:stat_price}, we show that when $\concepts$ is group-realizable with respect to $\groups$ and $\hypotheses$, empirical risk minimization over $\concepts$ provides the following guarantee with high probability\footnote{If both sides are divided by $\Pr\sbr{ x_n \in S }$, then the left-hand side becomes a conditional error rate $\Pr\Sbr{ \hat y_n \neq y_n \given x_n \in S }$, and the denominator on the right-hand side is $\Pr\sbr{ x_n \in S} \cdot n$, the expected number of data points in $U$ that come from group $S$.}:
\begin{equation}
  \Pr\sbr*{ x_n \in S \wedge \hat y_n \neq y_n }
  \leq O\del*{ \frac{d_{\concepts|_S}  \log^2(n)}{n} }
  \quad \text{for all $S \in \groups$}
  .
  \label{eq:multigroup_pac}
\end{equation}
There is no explicit dependence on any complexity measure of $\groups$ in the bound.



The guarantee in \Cref{eq:multigroup_pac} can be interpreted as ensuring a bound on the statistical error rate on a group $S$ chosen by an adversary in a manner that may depend on the training data $(x_1,y_1),\dotsc,(x_{n-1},y_{n-1})$.
In an alternative formulation of multi-group learning, we may ask for a guarantee only against a weaker adversary that is "oblivious" to the training data: the adversary chooses the group $S$ based on the learning algorithm, but before the training data is drawn.
In this setup, prior work of \citet{ardeshir2026realizable} proved an error bound of $O\del{ (d_{\concepts|_S} \log(n))/n }$.

At face value,
the lack of any dependence on a complexity measure of $\groups$
in the bound from \Cref{eq:multigroup_pac}
may seem surprising.
The key fact underlying this phenomenon is that group-realizability becomes a stronger assumption for larger \(\groups\), as each group imposes a constraint on the
family of concepts that may label the data.
In the extreme case where \(\groups = 2^\domain\), the largest concept class group-realizable with respect to \(\hypotheses\) is \(\hypotheses\) itself.
Without group-realizability, the situation is very different, and we have excess error rates of the form shown in \Cref{eq:multigroup_agnostic}; see \citealp{tosh2022simple} for more discussion.


The main focus of the present paper is to seek analogous group-realizable multi-group learning guarantees in the transductive setting, which has not been previously investigated.
How do the error bounds in the multi-group setting compare to those achievable by single-group learners?


\section{Transductive learning and orientations of one-inclusion graphs}
\label{sec:orientations}

This \namecref{sec:orientations} sets up the graph-theoretic framework of~\citet{haussler1994predicting} for (both single-group and multi-group) transductive learning in which prediction strategies are orientations of a certain induced subgraph of the $n$-dimensional Boolean hypercube.

An \emph{orientation} $\phi \colon E \to V$ of an undirected graph $G = (V,E)$ is a function that maps each edge $e$ to an incident vertex $\phi(e) \in e$.
Graphically, $\phi$ produces a directed edge $(u,v)$ when $\phi(\set{u,v}) = u$.
The out-degree of a vertex $v \in V$ in a subgraph $H$ of $G$ (with edges $E(H)$) oriented by $\phi$ is
\begin{equation*}
  \outdeg_{H,\phi}(v) \coloneqq \card{\phi^{-1}(v) \cap E(H)} = \card{\Set{ e \in E(H) \given \phi(e) = v }} .
\end{equation*}

The \emph{$n$-dimensional Boolean hypercube $Q_n$} is the undirected graph with vertex set $\bits^n$ and edge set containing all pairs of vertices $\set{u, v}$ that differ in exactly one position.
Given $n$ points $U = (x_1,\dotsc,x_n) \in \domain^n$ from a domain $\domain$ and a concept class $\concepts \subseteq \bits^{\domain}$ defined over this domain, the \emph{one-inclusion graph $G_{\concepts|_U}$ of $\concepts$ with respect to $U$} is the induced subgraph of $Q_n$ corresponding to the vertex set $V(G_{\concepts|_U}) \coloneqq \Set{ (c(x_1),\dotsc,c(x_n)) \given c \in \concepts }$ of "behaviors" of $\concepts$ on $U$; let $E(G_{\concepts|_U})$ denote the edges in this subgraph.

Although we have regarded $\groups$ as a family of subsets of the domain $\domain$, it is clear that since $U \in \domain^n$ is fixed, we can identify each $S \in \groups$ with a subset of $[n]$.
So henceforth, we overload the symbol $S$ to represent the subset of indices $i' \in [n]$ with $x_{i'} \in S$, and thus also regard $\groups$ as a subset of $2^{[n]}$.
In fact, any induced subgraph of $Q_n$ can be regarded as a one-inclusion graph over the domain $\domain \coloneqq [n]$ for a concept class specified by the vertices of the induced subgraph.

\paragraph{Prediction strategies as orientations of a one-inclusion graph.}

We now explain the sense in which a transductive learner's prediction strategy corresponds to orientations of $G_{\concepts|_U}$ (whether in the single-group or multi-group setting).
After $U = (x_1,\dotsc,x_n) \in \domain^n$ and $\concepts$ are fixed, the one-inclusion graph $G_{\concepts|_U}$ is determined, and the transductive learning scenario proceeds as follows.
\begin{itemize}
  \item[1.] Nature chooses a vertex $y = (y_1,\dotsc,y_n) \in V(G_{\concepts|_U})$.
  \item[2.] The test point index $i \in [n]$ is chosen uniformly at random.
  \item[3.] The learner receives only the labels $(y_j)_{j \neq i}$.
  \item[4.] The learner makes a (possibly randomized) prediction $\hat y_i$ of the label $y_i$.
\end{itemize}
The labels $(y_j)_{j \in [n]}$ correspond to some vertex in $G_{\concepts|_U}$, but the learner only has $(y_j)_{j \neq i}$; these labels will be consistent with either one or two vertices in $G_{\concepts|_U}$.
\begin{itemize}
  \item If there is only one consistent vertex in $G_{\concepts|_U}$, then the label $y_i$ is already known to the learner.

  \item If there are two consistent vertices $u$ and $v$ (differing only in the $i$-th position, $u_i \neq v_i$, so $\set{u,v}$ is an edge in $G_{\concepts|_U}$), then the prediction of the learner corresponds to a choice of one of these vertices, i.e., an orientation of the edge $\set{u,v}$.
    By convention, if the learner predicts the label $u_i$ (say), then we say the learner's orientation $\phi$ of $G_{\concepts|_U}$ has $\phi(\set{u,v}) = v$.

\end{itemize}
Every choice of the target concept by Nature corresponds to some vertex in $G_{\concepts|_U}$, and every choice of $i \in [n]$ (the index of the test point) specifies a potential edge incident with Nature's chosen vertex.
Thus, a prediction strategy of the learner corresponds to an orientation of $G_{\concepts|_U}$.
And vice versa: an orientation $\phi$ of $G_{\concepts|_U}$ defines a prediction strategy for every choice of $i \in [n]$.

A learner that makes randomized predictions is specified by pairs of non-negative numbers $p_{\set{u,v},u}$ and $p_{\set{u,v},v}$ summing to one for every edge $\set{u,v} \in E(G_{\concepts|_U})$: with probability $p_{\set{u,v},v}$, the learner predicts the label consistent with vertex $u$; with remaining probability $1 - p_{\set{u,v},v} = p_{\set{u,v},u}$, it predicts the label consistent with vertex $v$.
We may think of the learner as having a randomized orientation $\phi$ in which $\phi(\set{u,v})$ is $u$ with probability $p_{\set{u,v},u}$ and is $v$ with probability $p_{\set{u,v},v}$.

\paragraph{Transductive error rates.}

With the identification of a transductive learner's prediction strategy and an orientation $\phi$ of the one-inclusion graph $G_{\concepts|_U}$, we can write the transductive error rate in terms of the out-degree of Nature's chosen vertex in the oriented graph.
Let $v \in V(G_{\concepts|_U})$ be the vertex chosen by Nature.
In the single-group setting, \citet{haussler1994predicting} showed:
\begin{equation}
  \label{eq:error_degree}
  \Pr\sbr*{ \hat y_i \neq y_i }
  = \frac1n \E\sbr*{ \outdeg_{G_{\concepts|_U},\phi}(v) } .
\end{equation}
Therefore, an orientation $\phi$ with uniformly small out-degrees yields a small transductive error rate in the worst case (over Nature's choice of $v$).
The optimal worst case out-degree is controlled by the maximum subgraph edge density of the graph~\citep{alon1992colorings}.
For any graph $G$, define
\begin{equation*}
  \rho(G) \coloneqq \max_{\substack{U \subseteq V(G) \\ \text{s.t.}\, U \neq \emptyset}} \frac{\card{E(G[U])}}{\card{U}}
\end{equation*}
where $G[U]$ denotes the induced subgraph of $G$ on vertices $U \subseteq V(G)$.
\begin{proposition}[\citealp{alon1992colorings,asilis2024regularization}]
  \label{prop:density}
  For any graph $G$,
  \begin{equation*}
    \min_\phi \max_{v \in V(G)} \E\sbr*{ \outdeg_{G,\phi}(v) } = \rho(G)
  \end{equation*}
  where the minimization is taken over randomized orientations of $G$.
\end{proposition}
\citeauthor{haussler1994predicting} showed that if $\concepts$ has VC dimension $d_{\concepts}$, then there is an orientation $\phi$ such that $\outdeg_{G_{\concepts|_U}, \phi}(v) \leq d_{\concepts}$ for all $v \in V(G_{\concepts|_U})$.

In the multi-group setting, the transductive error rate on group $S$ only counts a mistake when $i \in S$.
Define the \emph{$S$-relevant subgraph $G_{\concepts|_U}^S$ of $G_{\concepts|_U}$} to be the (spanning) subgraph that retains only edges $\set{u,v} \in E(G_{\concepts|_U})$ for which the unique $i' \in [n]$ with $u_{i'} \neq v_{i'}$ also satisfies $i' \in S$.
We have the following generalization of \Cref{eq:error_degree} due to \citet{haussler1994predicting} (whose proof we replicate for completeness).
\begin{proposition}
  If Nature chooses vertex $y \in V(G_{\concepts|_U})$, then the transductive error rate of the learner on a group $S \subseteq [n]$ is
  \begin{equation*}
    \Pr\sbr*{ x_i \in S \wedge \hat y_i \neq y_i }
    = \frac1n \E\sbr*{ \outdeg_{G_{\concepts|_U}^S,\phi}(y) }
  \end{equation*}
  where $\phi$ is the (possibly randomized) orientation of $G_{\concepts|_U}$ corresponding to the learner.
\end{proposition}
\begin{proof}
  By linearity of expectation and direct computation:
  \begin{align*}
    \Pr\sbr*{ x_i \in S \wedge \hat y_i \neq y_i }
    & = \frac1n \sum_{i'=1}^n \E\sbr*{ \ind{ i' \in S \wedge \hat y_{i'} \neq y_{i'} } } \\
    & = \frac1n \E\sbr[\bigg]{ \sum_{i'=1}^n \ind{ i' \in S \wedge \exists \set{u,y} \in E(G_{\concepts|_U}) \,\text{s.t.}\, u_{i'} \neq y_{i'} \wedge \phi(\set{u,y}) = y } } \\
    & = \frac1n \E\sbr[\bigg]{ \sum_{\set{u,y} \in E\del{G_{\concepts|_U}^S}} \ind{ \phi(\set{u,y}) = y } }
    = \frac1n \E\sbr*{ \outdeg_{G_{\concepts|_U}^S,\phi}(y) }
  \end{align*}
  as claimed.
\end{proof}

Suppose for the moment that $\dcs$ is small for all $S \in \groups$; this is the case, e.g., when $\concepts$ is group-realizable with respect to $\groups$ and $\hypotheses$ and $\dhyp$ is small.
Then, for each $S \in \groups$, the $S$-relevant subgraph $G_{\concepts|_U}^S$ decomposes, after fixing the labels outside $S$, into subgraphs that project into the one-inclusion graph $G_{\concepts|_{U_S}}$ of $\concepts$ with respect to $U_S \coloneqq (x_i)_{i \in S}$.
Consequently, an orientation of $G_{\concepts|_{U_S}}$ with maximum out-degree at most $\dcs$ can be used, fiber by fiber, to construct an orientation of $G_{\concepts|_U}^S$ with maximum out-degree at most $\dcs$.
So, considering any group $S \in \groups$ in isolation, it is possible to achieve small transductive error rate on group $S$.
But can this be made compatible with achieving small transductive error rates on other groups in $\groups$?

\paragraph{Price of multi-group transductive learning.}

We now sharpen the question from the end of \Cref{sec:multigroup}.
For each $S \in \groups$, define the (unnormalized) \emph{minimax transductive error rate on group $S$ (for $G_{\concepts|_U}$)} to be
\begin{equation*}
  \opt_S(G_{\concepts|_U}) \coloneqq \min_{\phi_S} \max_{v \in V(G_{\concepts|_U})} \outdeg_{G_{\concepts|_U}^S,\phi_S}(v) ,
\end{equation*}
where the $\min$ is taken over (deterministic) orientations $\phi_S$ of $G_{\concepts|_U}$, and define the \emph{price of multi-group transductive learning (for $G_{\concepts|_U}$)} to be
\begin{equation}
  \label{eq:price}
  \min_{\phi} \max_{S} \frac{\displaystyle \max_{v \in V(G_{\concepts|_U})} \E\sbr[\Big]{\outdeg_{G_{\concepts|_U}^S,\phi}(v)}}{\opt_S(G_{\concepts|_U})} ,
\end{equation}
with the $\min$ taken over randomized orientations $\phi$ of $G_{\concepts|_U}$, and with the $\max$ taken over groups $S \in \groups$ with non-empty $E(G_{\concepts|_U}^S)$ (and hence $\opt_S(G_{\concepts|_U}) > 0$).\footnote{%
  Restricting to deterministic orientations here can only change the price by at most a factor of two; see \Cref{sec:derandomization}.
  We also consider a more relaxed notion of the price in \Cref{sec:relaxed}.%
}
The fraction in \Cref{eq:price} is the ratio of the error rate of the multi-group learner (corresponding to $\phi$) on group $S$ to that of a learner who only needs to consider group $S$ in isolation.
We ask the following:
\emph{What is the price of multi-group transductive learning for the worst case choice of $G_{\concepts|_U}$ and $\groups$ (say, as a function of $n$ and $\card{\groups}$)?}

As a warm-up, consider the special case where the groups are disjoint.
Then the edge sets of the group-relevant subgraphs are disjoint.
In this case, it is not difficult to see that the price of multi-group transductive learning is exactly one, independent of how many groups there are. Thus, a worst case instance must have intersecting groups.

\section{Lower bound on the price of multi-group transductive learning} \label{sec:lb}

\begin{figure}[t]
  \centering
  \includegraphics[width=\linewidth]{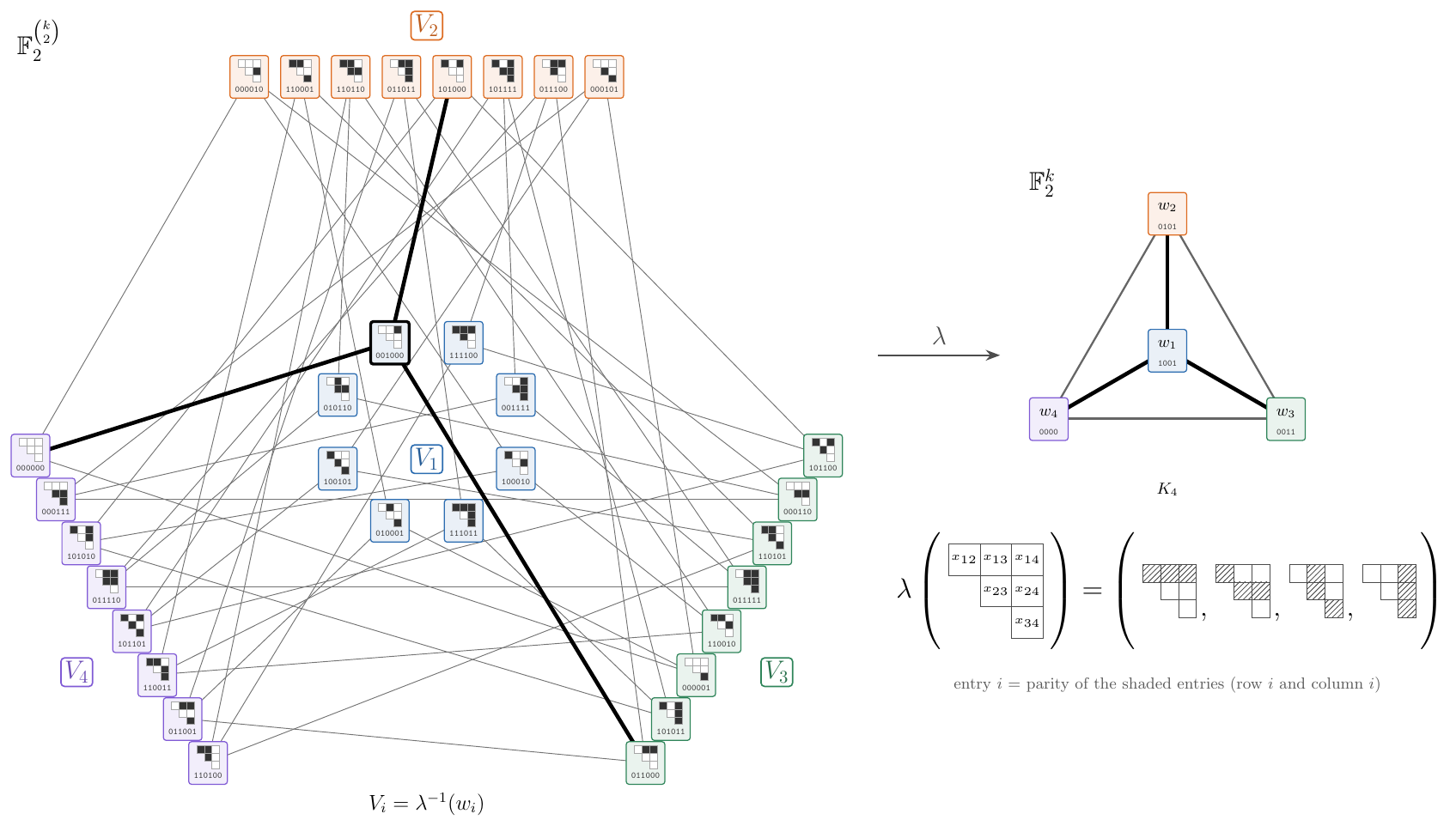}
  \caption{Illustration of the lower bound construction for \(k=4\) and \(n = {k \choose 2} = 6\). The subgraph of \(Q_{ n}\) (left) is a lift of a complete graph \(K_k\) (right). Each vertex on the left can be associated with strictly upper-triangular matrix, and the map \(\lambda\) records parities of subsections of this matrix. The pre-images of vectors \(w_i  \coloneqq e_i + e_k\) determine the parts \(V_i\).  
  }
  \label{fig:lowerbound}
\end{figure}

The following \namecref{thm:lb} implies a lower bound on the price of multi-group transductive learning.

\begin{theorem}
  \label{thm:lb}
  For any positive integers $n$ and $k \geq 3$ with $\binom{k}{2} \leq n$,
  there is a one-inclusion graph $G = (V,E)$ (as an induced subgraph of $Q_n$) and a family of groups $\groups \subseteq 2^{[n]}$
  with $\card{\groups} = k$
  such that the following holds:
  \begin{itemize}
    \item For each $S \in \groups$, we have $\opt_S(G) = 1$.

    \item For any randomized orientation $\phi$ of $G$, there is a vertex $v \in V$ and group $S \in \groups$ such that $\E\sbr{ \outdeg_{G^S,\phi}(v) } \geq (k-1)/2$.
  \end{itemize}
\end{theorem}

\Cref{thm:lb} shows that the price of multi-group transductive learning can be as large as $(\card{\groups}-1) / 2$, which is increasing linearly with $\card{\groups}$ up until $k \approx \sqrt{2n}$, upon which the lower bound on the price is about $\sqrt{n/2}$.
In contrast, the analogous "price of multi-group statistical learning" is never worse than $\poly\log(n)$ (as previously discussed in \Cref{sec:multigroup}; see also \Cref{sec:stat_price}).

The proof of \Cref{thm:lb} strikes a careful balance between constructing a one-inclusion graph with high overall edge density, and constructing (coordinate-defined) groups so that every vertex's neighborhood is in some group's relevant subgraph that has a low out-degree orientation.
Our one-inclusion graph is a particular lift of a complete graph on $k$ vertices;
each group's relevant subgraph turns out to be the disjoint union of stars, all of which are easily oriented to ensure out-degree at most one for all vertices.
Thus, every group is "easy" to orient in isolation, but because of the edge density of the overall graph, every orientation leaves some vertex with large out-degree, and that vertex's neighborhood---a star centered at the vertex---is wholly contained in some group.
The size of the star is also proportional to $k$, the number of groups.

Our lift of the complete graph on $k$ vertices requires $n \geq \binom{k}{2}$.
Although it is possible to obtain such lifts with only $n \asymp k$, the resulting one-inclusion graph would not have the other properties needed to imply the claimed lower bound on the price of multi-group transductive learning.
Indeed, the upper bounds in \Cref{sec:ub} show that our construction is quantitatively optimal up to constants.

\begin{proofof}{\Cref{thm:lb}}
  We may assume without loss of generality that $n = \binom{k}{2}$ for positive integer $k \geq 3$, since otherwise we can pad every vertex with zeros in the remaining $n-\binom{k}{2}$ coordinates.
  We identify the indices in $[n]$ with unordered pairs $\set{i,j} \in \kpairs \coloneqq \binom{[k]}{2}$, and we identify vertices of $Q_n$ as vectors from $\F_2^{\kpairs}$.
  Define the linear map $\lambda \colon \F_2^{\kpairs} \to \F_2^k$ by its action on the coordinate basis vectors $e_{\set{i,j}}$ for $\F_2^{\kpairs}$ for all $\set{i,j} \in \kpairs$:
  \begin{equation*}
    \lambda(e_{\set{i,j}}) \coloneqq e_i + e_j ,
  \end{equation*}
  where $e_1,\dotsc,e_k$ are the coordinate basis vectors for $\F_2^k$.
  So the image of $\lambda$ is the even parity subspace of $\F_2^k$.
  Further, define $w_i \coloneqq e_i + e_k \in \F_2^k$ for each $i \in [k]$ (so $w_k \coloneqq 0 \in \F_2^k$).
  Observe that each $w_i$ for $i \in [k]$ has even parity (and hence is in the image of $\lambda$), and
  \begin{equation*}
    w_i + w_j = e_i + e_j
    \quad \text{for all $\set{i,j} \in \kpairs$} .
  \end{equation*}

  Let $G = (V,E)$ be the subgraph of $Q_n$ induced by the vertex set $V \coloneqq V_1 \sqcup \dotsb \sqcup V_k$, where
  \begin{equation*}
    V_i \coloneqq \lambda^{-1}(w_i)
    \quad \text{for all $i \in [k]$} .
  \end{equation*}
  Note that each $V_i$ for $i \in [k]$ is non-empty.
  We now explicate the connectivity structure of this induced subgraph.
  For any $i_0 \in [k]$, $v \in V_{i_0}$, and $\set{i,j} \in \kpairs$, we have
  \begin{equation*}
    \lambda(v + e_{\set{i,j}})
    = \lambda(v) + \lambda(e_{\set{i,j}})
    = w_{i_0} + e_i + e_j
    = w_{i_0} + w_i + w_j ,
  \end{equation*}
  and therefore
  \begin{equation*}
    v + e_{\set{i,j}} \in V \iff i_0 \in \set{i,j} .
  \end{equation*}
  Indeed, if $i_0 = i$ (say), then $\lambda(v + e_{\set{i,j}}) = w_j$, hence $v + e_{\set{i,j}}  \in V_j$; meanwhile, if $i_0,i,j$ are distinct, then $\lambda(v + e_{\set{i,j}}) = w_{i_0} + w_i + w_j$ has Hamming weight three, so $\lambda(v + e_{\set{i,j}}) \not\in \{w_1,\dots, w_k\}$, so $v + e_{\set{i,j}} \not\in V$.
  This shows that $G$ is a lift of the complete graph on $k$ vertices, and in particular, each vertex has degree $k-1$.

  Define the family of groups by $\groups \coloneqq \Set{ S_i \given i \in [k]}$, where 
  \begin{equation*}
    S_i \coloneqq \Set*{ \set{i,j} \given j \in [k] \setminus \set{i} }
    \quad \text{for all $i \in [k]$} .
  \end{equation*}
  Consider group $S_{i_0}$ for some $i_0 \in [k]$.
  The edge set for the subgraph $G^{S_{i_0}}$ is
  \begin{equation*}
    E(G^{S_{i_0}}) = \Set*{ \set{ v, v + e_{\set{i_0,j}} } \given v \in V_{i_0}, j \in [k] \setminus \set{i_0} } ,
  \end{equation*}
  which is non-empty since $k\geq2$.
  So $G^{S_{i_0}}$ contains all edges of $G$ incident on vertices in $V_{i_0}$; and each vertex $v \in V_j$ for $j \neq i_0$ has exactly one neighbor in $G^{S_{i_0}}$, namely $v + e_{\set{i_0,j}}$.
  This means that $G^{S_{i_0}}$ is a disjoint union of "stars", with each vertex in $V_{i_0}$ being the center of a star with $k-1$ "spokes".

  For each $S \in \groups$, consider the orientation $\phi_S$ that directs each edge in $E(G^S)$ toward the center of the star containing that edge; this satisfies
  \begin{equation*}
    \max_{v \in V} \outdeg_{G^S,\phi_S}(v) \leq 1 .
  \end{equation*}
  Since $E(G^S)$ is non-empty, some vertex must have out-degree at least one, so
  \begin{equation*}
    \opt_S(G) = 1 .
  \end{equation*}

  Now consider any randomized orientation $\phi$ of $G$.
  By \Cref{prop:density}, we have
  \begin{equation*}
    \max_{v \in V}
    \E\sbr*{ \outdeg_{G,\phi}(v) }
    \geq \rho(G) \geq \frac{\card{E}}{\card{V}}
    .
  \end{equation*}
  So there is a vertex $v \in V$ such that
  \begin{equation*}
    \E\sbr*{ \card{\phi^{-1}(v)} }
    = \E\sbr*{ \outdeg_{G,\phi}(v) }
    \geq \frac{\card{E}}{\card{V}}
    = \frac{k-1}{2} ,
  \end{equation*}
  where the last equality follows from the fact that $G$ is $(k-1)$-regular.
  Fix such a vertex $v \in V$, and let $i_0 \in [k]$ be such that $v \in V_{i_0}$.
  Since $G^{S_{i_0}}$ contains all edges of $G$ incident on $v$, we have $\phi^{-1}(v) = \phi^{-1}(v) \cap E(G^{S_{i_0}})$.
  Therefore, for this vertex $v$ and group $S \coloneqq S_{i_0}$,
  \begin{equation*}
    \label{eq:outdeg_lb}
    \E\sbr*{ \outdeg_{G^S,\phi}(v) }
    = \E\sbr*{ \card{\phi^{-1}(v) \cap E(G^S)} }
    = \E\sbr*{ \card{\phi^{-1}(v)} }
    \geq \frac{k-1}{2}
  \end{equation*}
  as required.
\end{proofof}

\section{Methods for finding orientations with bounded group-relevant out-degrees} \label{sec:ub}

In this \namecref{sec:ub}, we give two methods for finding orientations of graphs that imply bounds on the price of multi-group transductive learning.
The first method ensures a bound on the price of $O(\sqrt{n})$, and the second ensures a bound of $O(\card{\groups})$.

\subsection{Peeling method}
\label{sec:peeling}

Our first method is a variant of the ``peeling algorithm'' of \citet{kuzmin2007unlabeled}; it repeatedly chooses a vertex and orients all (previously unoriented) incident edges away from the chosen vertex.
The order in which vertices are chosen is based on the out-degree bounds that are ultimately sought.

\begin{theorem}
  \label{thm:sqrtn}
  There is an algorithm that, given any induced subgraph $G$ of $Q_n$ and any family $\groups$ of subsets of $[n]$, constructs an orientation $\phi$ of $G$ such that for all $S \in \groups$ and all $v \in V(G)$,
  \begin{equation*}
    \outdeg_{G^S,\phi}(v) \leq 2\sqrt{n} \rho(G^S) .
  \end{equation*}
\end{theorem}

Returning to the setting of multi-group transductive learning with $G = G_{\concepts|_U}$, since $\rho(G^S)\leq \opt_S(G)$, \Cref{thm:sqrtn} (together with \Cref{prop:density}) implies the following upper bound on the price of multi-group transductive learning:
\begin{equation*}
  \min_{\phi} \max_{S} \frac{\displaystyle \max_{v \in V(G_{\concepts|_U})} \E\sbr[\Big]{\outdeg_{G_{\concepts|_U}^S,\phi}(v)}}{\opt_S(G_{\concepts|_U})} \leq 2\sqrt{n} .
\end{equation*}
(In fact, the orientation constructed by the peeling algorithm is deterministic; we conjecture that the factor of two can be removed using a randomized orientation.)

\subsection{Aggregation method}
\label{sec:aggregation}

To describe our second method---which has an analysis that is applicable to any graph, not just subgraphs of $Q_n$---we generalize the notion of a group to be any subset of the edges $E$.
For any $F \subseteq E$, let $G^F$ be the (spanning) subgraph of $G$ retaining only the edge set $F$.

\begin{theorem} \label{thm:lp}
  There is an algorithm that, given any graph $G = (V,E)$ and any family of edge subsets $\family \subseteq 2^E$, returns a randomized orientation $\phi$ of $G$ such that for all $F \in \family$ and all $v \in V$,
  \begin{equation}
    \label{eq:expected_outdeg_bound}
    \E\sbr*{ \outdeg_{G^F,\phi}(v) } \leq \card{\family} \opt_F(G)
  \end{equation}
  where
  \begin{equation*}
    \opt_F(G) \coloneqq \min_{\phi} \max_{v \in V} \outdeg_{G^F,\phi}(v) \quad \text{for all $F \in \family$} .
  \end{equation*}
\end{theorem}

The algorithm relies on the network flow method of~\citet{haussler1994predicting} to construct group-specific orientations (for the corresponding group-relevant subgraph), and then aggregates them into a single (randomized) orientation for the entire graph.
The aggregation is somewhat inspired by the online "specialists" aggregation method of \citet{freund1997using}, a variant of which was used for multi-group online learning~\citep{blum2019advancing} and multi-group statistical learning (via online-to-batch conversion)~\citep{tosh2022simple}.
However, it is not clear how to use \citeauthor{freund1997using}'s method (or variants thereof) directly in our transductive setting.

Applying \Cref{thm:lp} to our setting of multi-group transductive learning, where $G$ is a one-inclusion graph $G_{\concepts|_U}$, and each $F \in \family$ corresponds to $E(G_{\concepts|_U}^S)$ for some group $S \in \groups$, we obtain an upper-bound on the price of multi-group transductive learning:
\begin{equation*}
  \min_{\phi} \max_{S} \frac{\displaystyle \max_{v \in V(G_{\concepts|_U})} \E\sbr[\Big]{\outdeg_{G_{\concepts|_U}^S,\phi}(v)}}{\opt_S(G_{\concepts|_U})} \leq \card{\groups} .
\end{equation*}

We defer the proof of Theorem \ref{thm:lp} to Appendix \ref{sec:ub-appendix}.

\section{Price of statistical multi-group learning} \label{sec:stat_price}

The closest statistical analogue to the transductive setting is the oblivious multi-group setting, where the group that the learner is evaluated on is chosen independently of the sample (one can check that this is so). The per-group price in error rate in this setting is no more than a factor of \(\log(n)\) per \citet{ardeshir2026realizable}. It was previously unknown whether, for the more difficult non-oblivious multi-group learning setting, the price enjoys a similarly mild bound.
In this \namecref{sec:stat_price}, we confirm that this price
lies in the range $\intcc{\Omega(\log(n)),O(\log^2(n))}$.
 

\paragraph{Upper bound.} 
We prove the following upper bound on the error rate of ERM on group-realizable concepts in Appendix~\ref{app:stat-upper}

\begin{theorem}\label{thm:stat-upper}
There is a universal constant $C_0$ such that the following holds. Let
$\concepts\subseteq\bits^\domain$, let $\groups\subseteq2^\domain$ be any family of groups, let
$P$ be any distribution on $\domain$, and let $c^*\in\concepts$. For every $n\ge2$ and
$\delta\in\intoc{0,\tfrac12}$, with probability at least $1-\delta$ over
$x_1,\dots,x_{n-1}\sim P$ i.i.d., every $\hat c\in\concepts$ consistent with
$(x_j,c^*(x_j))_{j<n}$ satisfies
\[
  \Pr\sbr{x_n\in S\land\hat c(x_n)\neq c^*(x_n)}\;\le\;C_0\cdot\frac{\dcs\log^2(n)+\log(1/\delta)}{n}
  \qquad\text{for all } S\in\groups.
\]
\end{theorem}
 
We briefly describe the argument. The standard proof of the single-group ERM bound controls, via the ghost-sample trick and the Sauer--Shelah--Perles lemma, the number of ways the error sets $\Set{\hat c \neq c^*}$ for $\hat c\in\concepts$, can intersect a sample. In the multi-group setting the relevant sets are $S\cap\Set{\hat c\neq c^*}$ for \((\hat c, S) \in \concepts \times  \groups\), and this
family can easily inherit an infinite VC dimension from \(\groups\) \citep{ardeshir2026realizable}. The key remedy is to record these group-specific error sets as \emph{partial
concepts}, which in addition to \(\{0,1\}\) predictions can output a \(\star\) denoting an undefined prediction \citep{alon2022theory}. To each pair $( c,S) \in \concepts \times \groups$ we associate the partial concept 
\[ q_{c, S}(x) = \begin{cases}
    1 & x \in S\cap\set{  c\neq c^*}\\
    0 & x \in \set{  c =  c^*} \\
     \star & x \in \set{  c\neq c^*}\setminus S
\end{cases}.\]
Note that the VC dimension of the class \(\{q_{c, S}\}_{(c,S) \in \concepts \times \groups}\) (size of the largest \(\{0,1\}\)-shattered set) is at most \(\max_S d_{\concepts|_S}\)
. The ghost-sample argument then goes through once the partial class restricted to the pooled sample is replaced by a finite total class containing a completion of each of its members. The quasipolynomial Sauer--Shelah--Perles lemma of \citet[Theorem~12]{alon2022theory} provides one of size $\exp(O(d\log^2 n))$, where \(d\) is the VC dimension of the partial concept class (compare this to the $\exp(O(d\log n))$-sized restriction to the sample given by the standard Sauer--Shelah--Perles lemma). The extra \(\log(n)\) cannot be removed for general partial classes \citep[Theorem~11]{alon2022theory}, so removing that factor in the context of this upper bound would have to exploit the specific structure of these error concepts.

\paragraph{Lower bound.}
While consistent learners incur an \(O(\log^2 n)\) price in the non-oblivious setting, \textit{every} learner must incur an \(\Omega(\log n)\) price. We prove this in Appendix~\ref{app:stat-lower}, Theorem~\ref{thm:stat-lower}. The construction takes
$\concepts$ to be all functions on $\mathbb N$ and $\groups$ to be all $d$-point subsets, with $P$
uniform on about $d/\epsilon$ points. The learner's prediction at each point it has not seen is effectively as good as random, and any $d$ wrong predictions form a group with error about $\epsilon$; avoiding this
requires seeing all but $O(d)$ of the points, a coupon-collector cost of
$\Omega\del{\frac d\epsilon\log\frac1\epsilon}$ samples. The argument crucially uses the adversarial choice
of the group: for any fixed group, the same instance is learnable at the single-group
rate. This is very different from the transductive lower bound of Section~\ref{sec:lb}, where the
difficulty comes from the edge density of the one-inclusion graph rather than from unseen points.

\section{Discussion}
\label{sec:discussion}

The main result of this paper shows that the difficulty of multi-group transductive learning relative to standard (single-group) transductive learning can be much greater than that for statistical learning. One key difference between the learning settings that contributes to this separation has to do with the adversary's control over the data: in the transductive setting, the learner's adversary picks the dataset, whereas in the statistical setting, the adversary picks the distribution from which the data points are i.i.d. sampled. This observation suggests that the online setting may also suffer a worse ``price of multi-group learning'' than the statistical setting. We are not aware of lower bounds in online multi-group settings, although some upper bounds (based on randomized algorithms) have been given by \citet{blum2019advancing}.
Closing the gaps in the statistical and online settings is an interesting open problem.
Finally, it would also be interesting to explore multi-group extensions of agnostic transductive learning~\citep{long1998complexity,asilis2024regularization}.

\section{AI Disclosure}
\label{sec:ai_disclosure}

Commercial AI tools were used to come up with the proofs of Theorems~\ref{thm:stat-upper} and \ref{thm:stat-lower}.
In particular, for Theorem~\ref{thm:stat-upper}, they found the application of partial concepts~\citep{alon2022theory} that we had not previously seen.
Commercial AI tools were also used to generate the TikZ code for Figure~\ref{fig:lowerbound}.
AI tools were not used for the other results except to check their correctness.


\section*{Acknowledgments}
We thank Navid Ardeshir and Jingwen Liu for helpful conversations in an initial phase of this project. We also thank an anonymous COLT 2026 reviewer for pointing out a critical error in a submission that ultimately led us to pursue this line of inquiry. This work was supported in part by the NSF under grant DMS-2502259 and by the ONR under grant N00014-24-1-2700.

\bibliographystyle{plainnat}
\bibliography{references}

\appendix

\crefalias{section}{appendix}
\crefalias{subsection}{appendix}

\section{Additional related work}
\label{sec:related}

Statistical multi-group learning was initially studied by \citet{rothblum2021multi} and later by \citet{tosh2022simple}, building on the prior development in the online setting by \citet{blum2019advancing} (with later algorithmic developments by \citet{acharya2023oracle} and \citet{deng2024group}).
\citet{tosh2022simple} introduced the group-realizability assumption discussed in \Cref{sec:multigroup}, which was also studied by \citet{ardeshir2026realizable}.
The relationship between multi-group learning and other learning frameworks such as multicalibration~\citep{hebert2018multicalibration}, omniprediction~\citep{gopalan2022omnipredictors}, and multiaccuracy~\citep{kim2019multiaccuracy} has been explored by \citet{haghtalab2023unifying} and \citet{balakrishnan2026panprediction}.
Sample complexity of multicalibration and omniprediction has also been studied in recent works~\citep{gibbs2025sample,collina2026sample}.
Statistical multi-group active learning was studied by~\citet{rittler2023agnostic}, where the goal is to improve the label complexity through adaptivity.
The focus of our paper is multi-group transductive learning, which has not been previously studied.

\citet{haussler1994predicting} introduced and analyzed a learning algorithm based on the one-inclusion graph for the (single-group, leave-one-out) transductive learning setting, and showed that it also yields a learning algorithm for various statistical settings (although its error rate may not have optimal dependence on the "confidence parameter"~\citep{aden2023one}).
This algorithm has been extended for other learning settings, such as multi-class learning~\citep{rubinstein2009shifting,asilis2024regularization}, real-valued prediction~\citep{bartlett1998prediction}, and agnostic learning~\citep{long1998complexity,asilis2024regularization} to name a few.
The upper bounds in our paper continue this line of work by providing extensions to multi-group learning.

The presence of large stars in the one-inclusion graph has been identified as a barrier for various statistical learning settings~\citep{hanneke2015minimax,hanneke2016refined,aden2023one}.
A single large star in the one-inclusion graph is not necessarily a barrier in (multi-group) transductive learning, since stars do not have high edge density.
The lower bound in our paper constructs a one-inclusion graph with high overall edge density in which every vertex is the center of some large star in a subgraph defined by a subset of dimensions.

\section{Methods for finding orientations with bounded group-relevant out-degrees} \label{sec:ub-appendix}

\subsection{Peeling method}

\begin{proofof}{\Cref{thm:sqrtn}}
  The method for constructing the orientation $\phi$ is as follows.
  Start with $H \coloneqq G$.
  While $V(H)$ is non-empty, choose any $v \in V(H)$ satisfying
  \begin{equation*}
    \deg_{H^S}(v) \leq 2\sqrt{n} \rho(H^S) \quad \forall S \in \groups ,
  \end{equation*}
  and then update $H$ by removing the chosen $v$ and all incident edges.
  (Here, $\deg_{H^S}(v)$ refers to the undirected degree of vertex $v$ in $H^S$.)
  By \Cref{lem:low_degree}, such choices are possible at every step.
  For all edges $\set{u,v} \in E(G)$, set $\phi(\set{u,v}) \coloneqq u$ if $u$ is removed earlier than $v$ in this iterative process.

  Consider any $S \in \groups$ and any $v \in V(G)$.
  Let $H$ be the induced subgraph of $G$ just prior to the time $v$ is removed.
  Then the neighbors of $v$ in $H$ correspond to edges $\set{v,v'} \in E(H)$ with $\phi(\set{v,v'}) = v$, so $\outdeg_{G^S,\phi}(v) = \deg_{H^S}(v)$.
  Also, since $H$ is an induced subgraph of $G$, we have $\rho(H^S) \leq \rho(G^S)$.
  Therefore
  \begin{equation*}
    \outdeg_{G^S,\phi}(v) = \deg_{H^S}(v) \leq 2\sqrt{n} \rho(H^S) \leq 2\sqrt{n} \rho(G^S)
  \end{equation*}
  as required.
\end{proofof}

The proof of \Cref{thm:sqrtn} crucially relies on the following \namecref{lem:low_degree}.
\begin{lemma}
  \label{lem:low_degree}
  For every non-empty induced subgraph $H$ of $Q_n$ and every family $\groups$ of non-empty subsets of $[n]$, there exists a vertex $v \in V(H)$ such that, for all $S \in \groups$,
  \begin{equation*}
    \deg_{H^S}(v) \leq 2\sqrt{n} \rho(H^S) .
  \end{equation*}
\end{lemma}

Before giving the proof of \Cref{lem:low_degree}, we define the following notation.
For a subgraph $H$ of $Q_n$ and any vertex $v \in V(H)$, define
\begin{equation*}
  C_H(v) \coloneqq \Set{ i \in [n] \given \exists v^{\oplus i} \in V(H) }, 
\end{equation*}
where $v^{\oplus i}$ denotes the vertex of $Q_n$ that differs from $v$ only in the $i$-th component.

\begin{lemma}
  \label{lem:degree_bound}
  Let $H$ be any induced subgraph of $Q_n$ with no isolated vertices (so $C_H(v) \neq \emptyset$ for all $v \in V(H)$).
  For any choices of non-empty $T(v) \subseteq C_H(v)$ for each $v \in V(H)$,
  \begin{equation*}
    \sum_{v \in V(H)} \deg_{H^{T(v)}}(v) \leq 2\sqrt{n} \sum_{v \in V(H)} \rho(H^{T(v)}) .
  \end{equation*}
\end{lemma}

\begin{proofof}{\Cref{lem:low_degree}}
  The claim is trivially true if there is an isolated vertex in $H$.
  So assume there are no isolated vertices in $H$.
  Suppose, for sake of contradiction, that for all $v \in V(H)$, there exists $S(v) \in \groups$ such that
  \begin{equation*}
    \deg_{H^{S(v)}}(v) > 2\sqrt{n} \rho(H^{S(v)}) .
  \end{equation*}
  Then, summing over all vertices $v \in V(H)$,
  \begin{equation}
    \label{eq:contradiction}
    \sum_{v \in V(H)} \deg_{H^{S(v)}}(v) > 2\sqrt{n} \sum_{v \in V(H)} \rho(H^{S(v)}) .
  \end{equation}
  For each $v \in V(H)$, let $T(v) \coloneqq S(v) \cap C_H(v)$, so
  \begin{equation*}
    \deg_{H^{S(v)}}(v) = \deg_{H^{T(v)}}(v) ;
  \end{equation*}
  and since $\deg_{H^{S(v)}}(v) > 0$, we also have $T(v) \neq \emptyset$.
  Since $T(v) \subseteq S(v)$, the definition of $\rho$ implies
  \begin{equation*}
    \rho(H^{T(v)}) \leq \rho(H^{S(v)}) .
  \end{equation*}
  Therefore, \Cref{eq:contradiction} implies
  \begin{equation*}
    \sum_{v \in V(H)} \deg_{H^{T(v)}}(v) > 2\sqrt{n} \sum_{v \in V(H)} \rho(H^{T(v)}) ,
  \end{equation*}
  which is a contradiction of \Cref{lem:degree_bound}.
\end{proofof}

\begin{proofof}{\Cref{lem:degree_bound}}
  We may assume without loss of generality that $V(H) \neq \emptyset$.
  For each $i \in [n]$, let $m_i$ be the number of edges $\set{u,v}$ in $H$ with $u_i \neq v_i$, and let $n_i$ be the number of vertices $v \in V(H)$ with $i \in T(v)$.

  We begin with a few observations.
  For each $i \in [n]$,
  \begin{equation}
    \label{eq:matching}
    n_i
    = \sum_{v \in V(H)} \ind{i \in T(v)}
    \leq \sum_{v \in V(H)} \ind{i \in C_H(v)}
    = \sum_{v \in V(H)} \deg_{H^{\set{i}}}(v)
    = 2 m_i
    ,
  \end{equation}
  where we use the assumption that $T(v) \subseteq C_H(v)$ for all $v \in V(H)$, as well as the fact that the edges of the form $\set{v,v^{\oplus i}}$ form a matching in $H$.
  Moreover, again using the assumption $T(v) \subseteq C_H(v)$ for all $v \in V(H)$,
  \begin{equation}
    \label{eq:lhs}
    \sum_{i=1}^n n_i
    = \sum_{i=1}^n \sum_{v \in V(H)} \ind{i \in T(v)}
    = \sum_{v \in V(H)} \sum_{i=1}^n \ind{i \in T(v)}
    = \sum_{v \in V(H)} \deg_{H^{T(v)}}(v)
    .
  \end{equation}
  Finally, for all $v \in V(H)$, the definition of $\rho(H^{T(v)})$ and the assumptions imply
  \begin{align}
    \rho(H^{T(v)})
    & \geq \frac{\card{E(H^{T(v)})}}{\card{V(H)}}
    \label{eq:rho_lb1}
    \\
    \intertext{(where we use $V(H) \neq \emptyset$ and $U = V(H)$), and}
    \rho(H^{T(v)})
    & \geq \frac12
    \label{eq:rho_lb2}
  \end{align}
  Indeed, since $T(v)\neq\emptyset$ and $T(v)\subseteq C_H(v)$, there is some $i\in T(v)$ such that $v^{\oplus i}\in V(H)$; taking the set $\{v,v^{\oplus i}\}$ in the definition of $\rho(H^{T(v)})$ gives $\rho(H^{T(v)}) \geq 1/2$.

  Now we proceed to prove the claimed inequality.
  We have
  \allowdisplaybreaks
  \begin{align*}
    \del*{ \sum_{v \in V(H)} \deg_{H^{T(v)}}(v) }^2
    & = \del*{ \sum_{i=1}^n n_i }^2 && \text{(by \Cref{eq:lhs})} \\
    & \leq n \sum_{i=1}^n n_i^2 && \text{(by Cauchy-Schwarz)} \\
    & \leq 2n \sum_{i=1}^n m_i n_i && \text{(by \Cref{eq:matching})} \\
    & = 2n \sum_{i=1}^n m_i \sum_{v \in V(H)} \ind{i \in T(v)} \\
    & = 2n \sum_{v \in V(H)} \sum_{i \in T(v)} m_i \\
    & = 2n \sum_{v \in V(H)} \card{E(H^{T(v)})} \\
    & \leq 2n \sum_{v \in V(H)} \rho(H^{T(v)}) \card{V(H)} && \text{(by \Cref{eq:rho_lb1})} .
  \end{align*}
  Therefore
  \begin{equation}
    \label{eq:bound1}
    \sum_{v \in V(H)} \deg_{H^{T(v)}}(v) \leq \del*{2n \sum_{v \in V(H)} \rho(H^{T(v)}) \card{V(H)}}^{1/2}
  \end{equation}
  Furthermore, by \Cref{eq:rho_lb2}, we have
  \begin{equation}
    \label{eq:bound2}
    \sum_{v \in V(H)} \rho(H^{T(v)}) \geq \frac{\card{V(H)}}{2} .
  \end{equation}
  Combining \Cref{eq:bound1,eq:bound2} gives the claimed inequality.
\end{proofof}

\subsection{Aggregation method}

\begin{proofof}{\Cref{thm:lp}}
  Without loss of generality, assume $E(G^F)$ is non-empty for all $F \in \family$ (and hence $\opt_F(G) > 0$ for all $F \in \family$), and also that $\bigcup_{F \in \family} F = E$.
  The algorithm is as follows.
  For each $F \in \family$, obtain an orientation $\phi_F$ of $G^F$ achieving
  \begin{equation*}
    \max_{v \in V} \outdeg_{G^F,\phi_F}(v) = \opt_F(G)
  \end{equation*}
  using the network flow method of \citet{haussler1994predicting} on $G^F$.
  Let $X \coloneqq \Set{ (e,v) \given e \in E, v \in e }$, and write $F(e) \coloneqq \ind{e \in F}$ for every $e \in E$ and $F \in \family$.
  For each $(e,v) \in X$, define
  \begin{equation*}
    p_{e,v} \coloneqq
    \frac{
      \sum_{F \in \family} F(e) w_F \ind{\phi_F(e) = v}
    }{
      \sum_{F \in \family} F(e) w_F
    }
    ,
  \end{equation*}
  where we define the "weights" $\set{w_F}_{F \in \family}$ by
  \begin{equation}
    \label{eq:weights}
    w_F \coloneqq \frac1{\opt_F(G)} \quad \text{for each $F \in \family$} .
  \end{equation}
  (Technically, we must extend each $\phi_F$ to an orientation of all of $E$; but it is clear that the values on $E \setminus F$ do not matter.)

  Observe that for any $e \in E$,
  \begin{align*}
    \sum_{v \in e} p_{e,v}
    & = 
    \sum_{v \in e}
    \frac{
      \sum_{F \in \family} F(e) w_F \ind{\phi_F(e) = v}
    }{
      \sum_{F \in \family} F(e) w_F
    }
    \\
    & = 
    \frac{
      \sum_{F \in \family} F(e) w_F \sum_{v \in e} \ind{\phi_F(e) = v}
    }{
      \sum_{F \in \family} F(e) w_F
    }
    = 1 .
  \end{align*}
  Thus, we construct a randomized orientation $\phi$ of $G$ as follows: for each edge $\set{u,v} \in E$, set
  \begin{equation*}
    \phi(\set{u,v})
    \coloneqq
    \begin{cases}
      u & \text{with probability $p_{\set{u,v},u}$} ; \\
      v & \text{with probability $p_{\set{u,v},v}$} .
    \end{cases}
  \end{equation*}

  \allowdisplaybreaks
  It remains to show that $\sum_{e \ni v} F(e) p_{e,v} \leq \card{\family} \opt_F(G)$ for all $(F,v) \in \family \times V$, since this in turn implies \Cref{eq:expected_outdeg_bound} for all $(F,v) \in \family \times V$ as desired.
  Fix any $(F,v) \in \family \times V$.
  For any $e \in E$, define
  \begin{equation*}
    Z(e) \coloneqq \sum_{F' \in \family} F'(e) w_{F'} .
  \end{equation*}
  The terms in the sum defining $Z(e)$ are non-negative, so we have the inequality $Z(e) \geq F(e) w_F$, which rearranges to
  \begin{equation*}
    \frac{F(e)}{Z(e)} \leq \frac{1}{w_F}
    .
  \end{equation*}
  Therefore,
  \begin{align*}
    \sum_{e \ni v} F(e) p_{e,v}
    & =
    \sum_{e \ni v} F(e)
    \frac{
      \sum_{F' \in \family} F'(e) w_{F'} \ind{\phi_{F'}(e) = v}
    }{
      Z(e)
    }
    \\
    & =
    \sum_{F' \in \family} 
    \sum_{e \ni v}
    \frac{F(e)}{Z(e)} \cdot
    F'(e) w_{F'} \ind{\phi_{F'}(e) = v}
    \\
    & \leq
    \sum_{F' \in \family} 
    \sum_{e \ni v}
    \frac1{w_F} \cdot
    F'(e) w_{F'} \ind{\phi_{F'}(e) = v}
    \\
    & \leq \sum_{F' \in \family} \frac1{w_F} \cdot w_{F'} \opt_{F'}(G)
    \\
    & = \opt_F(G) \, \card{\family} .
  \end{align*}
  The last inequality follows from the fact that $\phi_{F'}$ satisfies $\max_{v \in V} \outdeg_{G^{F'},\phi_{F'}}(v) = \opt_{F'}(G)$, and the last equality uses \Cref{eq:weights}.
\end{proofof}

\section{Derandomization}
\label{sec:derandomization}

Suppose $\phi$ is a randomized orientation of $G = (V,E)$, and define
\begin{equation*}
  p_{e,v} \coloneqq \Pr\sbr{ \phi(e) = v } \quad \text{for all $e \in E$ and $v \in e$} .
\end{equation*}
Define the (non-random) orientation $\psi$ of $G$ as follows: for each $e \in E$, assign $\psi(e)$ to a vertex $v \in e$ with $p_{e,v} \geq 1/2$, breaking ties arbitrarily.
For any $e \in E$ and $v \in e$,
\begin{equation*}
  \ind{ \psi(e) = v } \leq 2p_{e,v} .
\end{equation*}
Therefore, for any $F \subseteq E$ and $v \in V$,
\begin{equation*}
  \outdeg_{G^F,\psi}(v)
  = \sum_{e \ni v} \ind{e \in F} \ind{\psi(e) = v}
  \leq \sum_{e \ni v} \ind{e \in F} 2p_{e,v}
  = 2 \E\sbr*{ \outdeg_{G^F,\phi}(v) } .
\end{equation*}

\section{Relaxed price of multi-group transductive learning}
\label{sec:relaxed}

The instance exhibited in the lower bound from \Cref{thm:lb} has a constant (unnormalized) minimax transductive error rate on each group; in particular, $\opt_S(G) = 1$ for all $S \in \groups$.
Consequently, the lower bound is compatible with the following possibility: for any induced subgraph $G$ of $Q_n$ and any family of groups $\groups \subseteq 2^{[n]}$, there is a randomized orientation $\phi$ of $G$ such that
\begin{equation}
  \label{eq:relaxed}
  \max_{S} \frac{\displaystyle\max_{v \in V(G)} \E\sbr[\Big]{ \outdeg_{G^S,\phi}(v) }}{\opt_S(G) + O\del*{\min\set{\card{\groups},\sqrt{n}}} } = O(1)
\end{equation}
with the $\max$ taken over groups $S \in \groups$ with non-empty $E(G^S)$.
In this \namecref{sec:relaxed}, we extend the construction in \Cref{thm:lb} to rule out such a guarantee.

Start with the same graph $G$ and family of groups $\groups$ in \Cref{thm:lb}.
For simplicity, assume $k$ (the number of groups) is a perfect square $k = \ell^2$ for a positive integer $\ell\geq2$, and $n = \binom{\ell^2}{2}$.
We form a new family of $\ell$ groups $\groups' \coloneqq \set{ S_1',\dotsc,S_{\ell}' }$ by merging some of the original groups from $\groups$ as follows: for each $t \in [\ell]$, let $I_t \coloneqq \set{ (t-1)\ell + 1, \dotsc, t\ell }$, and set $S_t' \coloneqq \bigcup_{i \in I_t} S_i$.

The second bullet in \Cref{thm:lb} continues to hold for the new family of $\groups'$: for any randomized orientation $\phi$ of $G$, there is a vertex $v \in V$ and group $S' \in \groups'$ such that $\E\sbr{ \outdeg_{G^{S'},\phi}(v) } \geq (k-1)/2$.
To see this, fix any randomized orientation of $G$ and any $v \in V$ and $S \in \groups$ with $\E\sbr{ \outdeg_{G^S,\phi}(v) } \geq (k-1)/2$.
Then for the new group $S' \in \groups'$ that contains $S$ as a subset, we have
$\phi^{-1}(v) = \phi^{-1}(v) \cap E(G^S) = \phi^{-1}(v) \cap E(G^{S'})$.
Therefore, we also have
\begin{equation*}
  \E\sbr{ \outdeg_{G^{S'},\phi}(v) } = \E\sbr{ \outdeg_{G^S,\phi}(v) } \geq \frac{k-1}{2} = \frac{\ell^2-1}{2} .
\end{equation*}

We now show a modification of the first bullet from \Cref{thm:lb} for the new family of groups $\groups'$.
Specifically, we show that $\opt_{S_t'}(G) \leq \ell$ for each $S_t' \in \groups'$.
To see this, we construct an orientation $\phi$ as follows: orient each edge in $E(G^{S_t'})$ toward any vertex in $V_t' \coloneqq \bigcup_{i \in I_t} V_i$.
(For edges between vertices in $V_t'$, orient the edge arbitrarily.)
In $G^{S_i}$ for $i \in I_t$, each vertex not in $V_i$ has one neighbor in $V_i$.
Therefore, in $G^{S_t'}$, each vertex not in $V_t'$ has $\ell$ neighbors in $V_t'$; and each vertex in $V_t'$ has $\ell-1$ neighbors in $V_t'$.
So in such an orientation, every vertex has out-degree at most $\ell$.

Therefore, for any randomized orientation $\phi$ of $G$, there is a vertex $v \in V$ and group $S' \in \groups'$ such that
\begin{equation*}
  \frac{\displaystyle\max_{v \in V(G)} \E\sbr[\Big]{ \outdeg_{G^{S'},\phi}(v) }}{\opt_{S'}(G) + O\del*{\min\set{\card{\groups'},\sqrt{n}}} }
  \geq \frac{\frac{\ell^2-1}{2}}{\ell + O\del[\Big]{\min\set[\Big]{\ell,\sqrt{\binom{\ell^2}{2}}}}}
  = \Omega(\ell) .
\end{equation*}
This implies that the guarantee in \Cref{eq:relaxed} is not possible.

\providecommand{\lab}[1]{\overset{\text{(#1)}}}

\section{Proofs for Section~\ref{sec:stat_price}}
\label{app:stat-price}

This appendix proves Theorem~\ref{thm:stat-upper} and states and proves the lower bound
(Theorem~\ref{thm:stat-lower}) described in Section~\ref{sec:stat_price}. Throughout,
$\concepts\subseteq\bits^\domain$ is a concept class, $\groups\subseteq2^\domain$ is a nonempty
family of groups, $c^*\in\concepts$ is the target, and $P$ is a distribution on $\domain$. We
write $m$ for the number of training points and $D_m\coloneqq((x_j,c^*(x_j)))_{j=1}^m$ with
$x_1,\dots,x_m\sim P$ i.i.d.; in the notation of Section~\ref{sec:stat_price}, $m=n-1$ and the
test point is $x_n$. For a learner $A$ that, after seeing $D_m$, outputs a (possibly randomized)
prediction $\hat y$ at a test point $x\sim P$, we write
\[
  \operatorname{err}_S(D_m)\coloneqq\Pr\Sbr{x\in S\land\hat y\neq c^*(x) \given D_m},
\]
the probability being over $x$ and the internal randomness of $A$. For a fixed classifier
$\hat c$ this is $\Pr_{x\sim P}\sbr{x\in S\land\hat c(x)\neq c^*(x)}$. For a class $\calQ$ of
total or partial concepts, $\operatorname{VC}(\calQ)$ denotes its VC dimension, so that
$\dcs=\operatorname{VC}(\concepts|_S)$. Logarithms are natural unless a base is shown.

\subsection{Partial concepts and disambiguation}

\begin{definition}[\citealp{alon2022theory}]\label{def:partial}
A \emph{partial concept} on a set $Z$ takes values in $\set{0,1,\star}$, with $\star$ meaning
unspecified. A family $\calQ$ of partial concepts \emph{shatters} $I\subseteq Z$ if every
$b\in\bits^I$ equals $f|_I$ for some $f\in\calQ$ (so no $\star$ is used on $I$);
$\operatorname{VC}(\calQ)$ is the largest size of a shattered set. A \emph{completion} of $f$ is
any $w\in\bits^Z$ agreeing with $f$ wherever $f\neq\star$, and a \emph{disambiguation} of $\calQ$
is a set of binary vectors containing a completion of every member of $\calQ$.
\end{definition}

We use the quasipolynomial Sauer--Shelah--Perles lemma of \citet{alon2022theory} as a black box.

\begin{lemma}[{\citealp[Theorem~12]{alon2022theory}}]\label{lem:disambiguation}
There is a universal constant $c_1\ge1$ such that for all integers $M\ge2$ and $d\ge1$, every
$\calQ\subseteq\set{0,1,\star}^M$ with $\operatorname{VC}(\calQ)\le d$ has a disambiguation of size
at most $M^{c_1d\log M}$.
\end{lemma}

(Theorem~12 of \citet{alon2022theory} is stated for $\operatorname{VC}(\calQ)=d$. If
$\operatorname{VC}(\calQ)=d'<d$, apply it with $d'$; if $d'=0$, the members of $\calQ$ never
disagree on a specified coordinate, so a single completion suffices.)

\subsection{Proof of Theorem~\ref{thm:stat-upper}}
\label{app:stat-upper}

Let $m\coloneqq n-1$, so $\frac n2\le m<n$. Groups with $\dcs=0$ have zero error, since all of
$\concepts$ agrees with $c^*$ on them, and groups with $\dcs\ge n$ satisfy the bound trivially once
$C_0\ge3$, since its right-hand side then exceeds $1$. The remaining groups are split by VC
dimension: for $j\ge1$ let
\[
  \groups_j\coloneqq\Set{S\in\groups\given 2^{j-1}\le\dcs<2^j},\qquad d_j\coloneqq2^j-1,\qquad
  \delta_j\coloneqq\frac{\delta}{2j^2},
\]
so that only $j\le J\coloneqq\floor{\log_2n}+1$ matter, $\sum_j\delta_j<\delta$, $\delta_j\le\frac14$,
and $d_j\le2\dcs$ for $S\in\groups_j$. Let $c_1$ be the constant of Lemma~\ref{lem:disambiguation},
put $C\coloneqq24c_1$, and set
\[
  \epsilon_j\coloneqq C\,\frac{d_j\log^2(n)+\log(1/\delta_j)}{n}.
\]
We show that for each $j$, with probability at least $1-\delta_j$,
\begin{equation}\label{eq:stratum}
  \Pr\sbr{x_n\in S\land\hat c(x_n)\neq c^*(x_n)}\le\epsilon_j
  \qquad\text{for every consistent }\hat c\in\concepts\text{ and every }S\in\groups_j .
\end{equation}
This is trivial if $\epsilon_j\ge1$, so assume $\epsilon_j<1$. Given~\eqref{eq:stratum}, a union
bound over $j$ shows that with probability at least $1-\delta$, every consistent $\hat c$ satisfies,
for every $j\le J$ and $S\in\groups_j$,
\[
  \Pr\sbr{x_n\in S\land\hat c(x_n)\neq c^*(x_n)}
  \le C\,\frac{2\dcs\log^2 n+\log(1/\delta)+\log(2j^2)}{n}
  \le 7C\,\frac{\dcs\log^2n+\log(1/\delta)}{n},
\]
using $\log(2j^2)\le\log(2(\log_2n+1)^2)\le5\log^2 n$ for $n\ge2$ and $\dcs\ge1$. This proves the
theorem with $C_0\coloneqq7C$.

It remains to prove~\eqref{eq:stratum} for a fixed $j$.

\emph{Error concepts.} For $c\in\concepts$ and $S\in\groups_j$ define the partial concept
\[
  q_{c,S}(x)\coloneqq\begin{cases}1,&x\in S\land c(x)\neq c^*(x),\\0,&c(x)=c^*(x),\\
  \star,&\text{otherwise},\end{cases}
  \qquad \calE\coloneqq\Set{q_{c,S}\given c\in\concepts,\ S\in\groups_j}.
\]
Then $c$ is consistent with $D_m$ if and only if $q_{c,S}(x_k)=0$ for all $k\le m$ (for any $S$),
and $\Pr_x\sbr{x\in S\land c(x)\neq c^*(x)}=\Pr_x\sbr{q_{c,S}(x)=1}$. Moreover
$\operatorname{VC}(\calE)\le d_j$. Indeed, if $\calE$ shatters $U\subseteq\domain$, the all-ones
pattern on $U$ is realized by some $q_{c_0,S_0}$, so $U\subseteq S_0$; and each $b\in\bits^U$ is
realized by some $q_{c,S}$, which forces $c|_U=c^*|_U\oplus b$ with $c\in\concepts$. Hence
$\concepts|_{S_0}$ shatters $U$, and $\card U\le d_{\concepts|_{S_0}}\le d_j$.

\emph{Ghost sample via a pool.} Let $x'_1,\dots,x'_m\sim P$ be an independent ghost sample.
Realize $(x_1,\dots,x_m,x'_1,\dots,x'_m)$ by drawing a pool $z_1,\dots,z_{2m}\sim P$ i.i.d.\ and,
independently, a uniformly random $m$-subset $T\subseteq[2m]$: the training sample lists
$(z_i)_{i\in T}$ and the ghost sample lists $(z_i)_{i\notin T}$, each in increasing order of $i$.
Given $T$ these are independent i.i.d.\ samples of size $m$, so the pair has the intended joint
law.

\emph{Completion set.} Fix the pool $z=(z_i)_{i\le2m}$. The family
$\calE|_z\coloneqq\Set{(q(z_i))_{i\le2m}\given q\in\calE}\subseteq\set{0,1,\star}^{2m}$ has VC
dimension at most $d_j$: two coordinates carrying the same point always agree, so a shattered set
of coordinates carries distinct points, which are then shattered by $\calE$.
Lemma~\ref{lem:disambiguation} with $M=2m\ge2$ gives a disambiguation $\calB_z\subseteq\bits^{2m}$
of $\calE|_z$, depending only on $z$, with
\begin{equation}\label{eq:Bz}
  \log\card{\calB_z}\le c_1d_j\log^2(2m).
\end{equation}

\emph{The chain.} Writing $\card w$ for the number of ones of $w\in\bits^{2m}$,
\begin{align*}
&\Pr_{D_m}\sbr[\Big]{\exists c\in\concepts,\,S\in\groups_j:\ c\text{ consistent with }D_m,\
   \Pr_x\sbr{x\in S\land c(x)\neq c^*(x)}>\epsilon_j}\\
&\ \lab{a}{\le}2\Pr_{D_m,x'}\sbr[\Big]{\exists c\in\concepts,\,S\in\groups_j:\ c\text{ consistent with }D_m,\
   \card{\Set{k\given x'_k\in S\land c(x'_k)\neq c^*(x'_k)}}\ge\tfrac{m\epsilon_j}2}\\
&\ \lab{b}{=}2\,\E_z\Pr_T\sbr[\Big]{\exists q\in\calE:\ q(z_i)=0\ \forall i\in T,\
   \card{\Set{i\notin T\given q(z_i)=1}}\ge\tfrac{m\epsilon_j}2}\\
&\ \lab{c}{\le}2\,\E_z\Pr_T\sbr[\Big]{\exists w\in\calB_z:\ w_i=0\ \forall i\in T,\ \card w\ge\tfrac{m\epsilon_j}2}\\
&\ \lab{d}{\le}2\,\E_z\sum_{w\in\calB_z:\,\card w\ge m\epsilon_j/2}\Pr_T\sbr{w_i=0\ \forall i\in T}
 \ \lab{e}{\le}\ 2\,\sup_z\card{\calB_z}\,2^{-m\epsilon_j/2}\\
&\ \lab{f}{\le}2\exp\del[\Big]{c_1d_j\log^2(2m)-\frac{\log2}{2}m\epsilon_j}
 \ \lab{g}{\le}\ \delta_j.
\end{align*}

\textbf{(a)} Fix $D_m$ in the first event and a witness $(c,S)$ (chosen measurably), and put
$p\coloneqq\Pr_x\sbr{x\in S\land c(x)\neq c^*(x)}>\epsilon_j$. The ghost sample is independent of
$D_m$, so the count $\card{\Set{k\given x'_k\in S\land c(x'_k)\neq c^*(x'_k)}}$ is
$\mathrm{Bin}(m,p)$. It falls below $m\epsilon_j/2<mp/2$ only if it deviates from its mean by more
than $mp/2$, which by Chebyshev's inequality has probability at most
$4mp(1-p)/(mp)^2<4/(m\epsilon_j)\le\frac12$. The last step uses
$m\epsilon_j\ge\frac n2\epsilon_j\ge\frac C2\del{\log^2 2+\log4}>8$. So, given any such $D_m$, the
event in line~(a) has conditional probability at least $\frac12$; averaging over $D_m$ gives~(a).

\textbf{(b)} Compute the probability in line~(a) under the pool construction. There, $c$ is
consistent with $D_m$ if and only if $q_{c,S}(z_i)=0$ for all $i\in T$, and the ghost points
$x'_k$ with $x'_k\in S\land c(x'_k)\neq c^*(x'_k)$ are exactly the $z_i$, $i\notin T$, with
$q_{c,S}(z_i)=1$. Since $z$ and $T$ are independent, the probability equals
$\E_z\Pr_T\sbr{\,\cdot\,}$.

\textbf{(c)} Let $q\in\calE$ witness the event in line~(b). Its restriction $(q(z_i))_{i\le2m}$
lies in $\calE|_z$, so it has a completion $w\in\calB_z$, i.e.\ $w_i=q(z_i)$ whenever
$q(z_i)\neq\star$. On $T$ we have $q(z_i)=0$, so $w_i=0$ for $i\in T$; off $T$, $w_i=1$ wherever
$q(z_i)=1$, so $\card w\ge\card{\Set{i\notin T\given q(z_i)=1}}\ge m\epsilon_j/2$. Hence $w$
witnesses the event in line~(c).

\textbf{(d)} Union bound over the finite set $\calB_z$. This is legitimate because, for fixed
$z$, the set $\calB_z$ is deterministic, while $T$ is drawn independently of $z$.

\textbf{(e)} A uniformly random $m$-subset of $[2m]$ avoids $k$ given coordinates with
probability
\[
  \binom{2m-k}{m}\Big/\binom{2m}{m}=\prod_{i<k}\frac{m-i}{2m-i}\le2^{-k}
\]
(which is $0$ for $k>m$). Each summand is therefore at most $2^{-\card w}\le2^{-m\epsilon_j/2}$,
and there are at most $\card{\calB_z}$ summands.

\textbf{(f)} Insert~\eqref{eq:Bz} and write $2^{-m\epsilon_j/2}=\exp(-\tfrac{\log 2}{2}m\epsilon_j)$.

\textbf{(g)} Since $\log(2m)\le\log(2n)\le2\log n$ for $n\ge2$, we have
$c_1d_j\log^2(2m)\le4c_1d_j\log^2n$. Since $m\ge\frac n2$ and $\frac{C\log2}{4}=6c_1\log2\ge4c_1$,
\[
  \frac{\log2}{2}m\epsilon_j\ \ge\ \frac{C\log2}{4}\del[\big]{d_j\log^2n+\log(1/\delta_j)}
  \ \ge\ 4c_1d_j\log^2n+2\log(1/\delta_j).
\]
Hence the exponent in line~(f) is at most $-2\log(1/\delta_j)\le-\log(2/\delta_j)$, using
$\delta_j\le\frac12$, and the right-hand side of~(f) is at most $\delta_j$.

\subsection{Lower bound}
\label{app:stat-lower}

\begin{theorem}\label{thm:stat-lower}
For every $d\ge1$ there are $\concepts\subseteq\bits^\domain$ and $\groups\subseteq2^\domain$
with $\dcs=d$ for all $S\in\groups$ (and $\concepts$ group-realizable with respect to $\groups$
and $\hypotheses\coloneqq\concepts$) such that the following holds for every $n\ge d$. There is a
distribution $P$ on $\domain$ such that for every learner $A$ there exists $c^*\in\concepts$ with
\[
  \Pr\sbr[\Big]{\,\exists S\in\groups:\ \Pr\sbr{x_n\in S\land\hat y_n\neq y_n}\ \ge\ \frac{d}{1024\,n}\log\frac nd\,}
  \ \ge\ \frac12 ,
\]
where the outer probability is over the training data.
\end{theorem}

The learner $A$ may be improper and randomized. Given $D_m$ and a test point $x$, its prediction
$\hat y$ is a random bit whose distribution depends on $c^*$ only through the labels in $D_m$;
this is all we use about $A$.

\begin{lemma}\label{lem:lower-eps}
Let $d\ge1$, $\domain\coloneqq\mathbb N$, $\concepts\coloneqq\bits^\domain$ and
$\groups\coloneqq\Set{S\subseteq\domain\given\card S=d}$. Let $\epsilon\in\intoo{0,\frac1{512}}$
and
\begin{equation}\label{eq:m-lower}
  m\;\le\;\frac{d}{16\epsilon}\log\frac1{512\epsilon}.
\end{equation}
Let $N\coloneqq\ceil{d/(8\epsilon)}$ and let $P$ be the uniform distribution on $[N]$. Then for
every learner $A$ there exists $c^*\in\concepts$ with
\[
  \Pr_{D_m}\sbr[\Big]{\max_{S\in\groups}\operatorname{err}_S(D_m)\ge\epsilon}>\frac12 .
\]
\end{lemma}

\begin{proof}
\emph{The instance.} Each point of $[N]$ has mass $p\coloneqq1/N$. Since
$\frac{d}{8\epsilon}\le N\le\frac{d}{8\epsilon}+1$ and $\epsilon<\frac1{64}$,
\begin{equation}\label{eq:p-bound}
  \frac{7\epsilon}{d}\ <\ \frac{8\epsilon}{d}\cdot\frac{1}{1+8\epsilon}\ \le\ p\ \le\ \frac{8\epsilon}{d}.
\end{equation}
Draw $c^*$ uniformly at random from the (finitely many) functions in $\concepts$ that vanish off
$[N]$.

\emph{Error probabilities at unseen points.} Call $x$ \emph{unseen} if
$x\notin\set{x_1,\dots,x_m}$, and let $$M\coloneqq\card{\Set{x\in[N]\given x\text{ unseen}}}.$$ For
$x\in[N]$ let
\[
  e(x)\coloneqq\Pr\Sbr{\hat y\neq c^*(x)\given D_m,\ \text{test point }x}\in\intcc{0,1},
\]
so that $\operatorname{err}_S(D_m)=p\sum_{x\in S}e(x)$ for $S\subseteq[N]$. Let
$a(x)\coloneqq\Pr\Sbr{\hat y=0\given D_m,\ \text{test point }x}$, a function of $D_m$ and $x$
only. Then $e(x)=a(x)$ if $c^*(x)=1$ and $e(x)=1-a(x)$ if $c^*(x)=0$. Conditionally on $D_m$, the
values $c^*(x)$ at the unseen points are i.i.d.\ uniform bits, since they are independent of the
labels in $D_m$. Hence, conditionally on $D_m$, the variables $e(x)$ for unseen $x$ are
independent, take values in $\intcc{0,1}$, and have mean $\frac12$ and variance at most $\frac14$.

\emph{The chain.} With $E\coloneqq\sum_{x\in M}e(x)$,
\begin{align*}
&\inf_A\ \max_{c^*\in\concepts}\ \Pr_{D_m}\sbr[\Big]{\max_{S\in\groups}\operatorname{err}_S(D_m)\ge\epsilon}\\
&\quad \ge \ \inf_A\ \E_{c^*}\Pr_{D_m}\sbr[\Big]{\max_{S\in\groups}\operatorname{err}_S(D_m)
\ge\epsilon} \\
&\quad \lab{a}{\ge}\ \inf_A\ \Pr_{c^*,D_m}\sbr[\big]{M\ge16d,\ E\ge\tfrac M4}\\
\text{(tower property)} &\quad =\ \inf_A\ \E_{D_m}\sbr[\Big]{\ind{M\ge16d}\cdot\Pr_{c^*}\Sbr[\big]{E\ge\tfrac M4\given D_m}}
\\
&\quad \lab{b}{\ge}\ \E_{D_m}\sbr[\Big]{\ind{M\ge16d}\del[\Big]{1-\frac4M}}\\
&\quad\lab{c}{\ge}\ \del[\Big]{1-\frac1{4d}}\Pr_{x_1,\dots,x_m}\sbr{M\ge16d}
\\
&\quad \lab{d}{\ge}\ \frac34\cdot\frac{35}{36}\ >\ \frac12 .
\end{align*}
Thus, every learner has some $c^*$ for which the claimed
probability exceeds $\frac12$.

\textbf{(a)} On the event, $\max_S\operatorname{err}_S(D_m)\ge\epsilon$. Indeed, let $S$ consist
of the $d$ unseen points of $[N]$ with the largest values of $e$, which is possible as $M\ge d$.
The average of the $d$ largest of $M$ numbers is at least their overall average, so
by~\eqref{eq:p-bound}
\[
  \operatorname{err}_S(D_m)=p\sum_{x\in S}e(x)\ \ge\ p\cdot\frac dM\cdot E\ \ge\ \frac{pd}{4}\ >\ \frac{7\epsilon}{4}\ >\ \epsilon .
\]


\textbf{(b)} Given $D_m$ with $M\ge1$, $E$ is a sum of $M$ independent $\intcc{0,1}$-valued
variables with mean $\frac12$, so $\E\Sbr{E\given D_m}=\frac M2$ and
$\operatorname{Var}\Sbr{E\given D_m}\le\frac M4$. By Chebyshev's inequality,
$\Pr\Sbr{E<\frac M4\given D_m}\le\Pr\Sbr{\abs{E-\frac M2}>\frac M4\given D_m}\le\frac{M/4}{(M/4)^2}=\frac4M$.
This bound does not depend on $A$, so the infimum disappears.

\textbf{(c)} On $\set{M\ge16d}$ we have $1-\frac4M\ge1-\frac1{4d}$.

\textbf{(d)} First, $1-\frac1{4d}\ge\frac34$. Second, $M=\sum_{x\in[N]}\ind{x\text{ unseen}}$, and
for $x\neq x'$ these indicators are negatively correlated, since
$\Pr\sbr{x,x'\text{ unseen}}=(1-2p)^m\le(1-p)^{2m}=\Pr\sbr{x\text{ unseen}}\Pr\sbr{x'\text{ unseen}}$.
So $\operatorname{Var}M\le\E M$. Using $1-p\ge e^{-2p}$ (as $p\le\frac12$), $p\le\frac{8\epsilon}d$,
$N\ge\frac{d}{8\epsilon}$ and~\eqref{eq:m-lower},
\[
  \E M=N(1-p)^m\ge N e^{-2pm}\ge\frac{d}{8\epsilon}\exp\del[\Big]{-\frac{16\epsilon m}{d}}
  \ge\frac{d}{8\epsilon}\exp\del[\Big]{-\log\frac1{512\epsilon}}=64d.
\]
Then $16d\le\frac14\E M$, so $M<16d$ forces $\abs{M-\E M}>\frac34\E M$, and Chebyshev's inequality
gives $\Pr\sbr{M<16d}\le\frac{16\operatorname{Var}M}{9(\E M)^2}\le\frac{16}{9\,\E M}\le\frac{16}{9\cdot64d}\le\frac1{36}$.
\end{proof}

\begin{proof}[Proof of Theorem~\ref{thm:stat-lower}]
Fix $d\ge1$ and take $\concepts$ and $\groups$ as in Lemma~\ref{lem:lower-eps}. Then
$\concepts|_S=\bits^S$, so $\dcs=d$ for every $S\in\groups$, and with
$\hypotheses\coloneqq\concepts$ the class is trivially group-realizable. Let $n\ge d$ and put
$u\coloneqq n/d\ge1$ and $m\coloneqq n-1$. If $u=1$, the threshold $\frac{d}{1024n}\log\frac nd$ is
$0$ and there is nothing to prove, so assume $u>1$ and set
\[
  \epsilon\coloneqq\frac{d}{1024\,n}\log\frac nd=\frac{\log u}{1024\,u}>0.
\]
We check the two hypotheses of Lemma~\ref{lem:lower-eps}. First, $\frac{\log u}{u}\le\frac1e$ for
all $u>1$, so $\epsilon\le\frac1{1024e}<\frac1{512}$. Second, substituting $\epsilon$,
\[
  \frac{d}{16\epsilon}\log\frac1{512\epsilon}=\frac{64\,n}{\log u}\log\frac{2u}{\log u},
\]
which is at least $n$ (and hence greater than $m$) if and only if
$\log\frac{2u}{\log u}\ge\frac{\log u}{64}$, i.e.\ $2u^{63/64}\ge\log u$. This holds for every
$u>1$: the minimum of $u^{a}/\log u$ over $u>1$ is $ae$, so $2u^{63/64}/\log u\ge\frac{63}{32}e>1$.
Let $P$ be the distribution of Lemma~\ref{lem:lower-eps} for this $\epsilon$. For every learner,
the lemma gives $c^*\in\concepts$ with
\[
  \Pr_{D_m}\sbr[\Big]{\max_{S\in\groups}\operatorname{err}_S(D_m)\ge\frac{d}{1024\,n}\log\frac nd}>\frac12 .
\]
Since $\operatorname{err}_S(D_m)=\Pr\sbr{x_n\in S\land\hat y_n\neq y_n}$ in the notation of
Section~\ref{sec:stat_price}, this proves the theorem.
\end{proof}

\begin{remark}
Adding a point of mass $2\epsilon$ to the construction, whose label the learner must guess when it
is unseen, shows in the same way that $\Omega\del{\frac1\epsilon\log\frac1\delta}$ samples are also
necessary for failure probability $\delta<\frac12$. Hence, in the $(\epsilon,\delta)$ form, the
sample complexity of (non-oblivious) multi-group learning is
$\Omega\del[\big]{\frac1\epsilon\del{d\log\frac1\epsilon+\log\frac1\delta}}$, compared with the upper bound
$O\del[\big]{\frac1\epsilon\del{d\log^2\frac d\epsilon+\log\frac1\delta}}$ implied by
Theorem~\ref{thm:stat-upper}.
\end{remark}

\end{document}